\documentclass[final,1p,times]{elsarticle}

\usepackage{hyperref}

\journal{Journal of \LaTeX\ Templates}

\usepackage{amssymb}
\usepackage[cmex10]{amsmath}
\usepackage{amsmath,amsfonts,amssymb,amsthm,dsfont,stmaryrd}
\usepackage{hyperref}
\usepackage{color}

\newtheorem{theorem}{Theorem}[section]

\newtheorem{lemma}{Lemma}[section]

\newtheorem{definition}{Definition}[section]

\def\s{s}

\def\R{\mathbb{R}}

\def\M{\mathcal{M}}

\def\m{\mathrm{m}}

\def\F{\mathcal{F}}

\def\w{\omega}

\def\x{\mathrm{x}}

\def\ICA{ICA$_{SGG}$}
\def\D{\mathcal{D}}

\def\z{\mathrm{z}}

\def\1{\mathds{1}}

\def\for{\mbox{  for }}

\def\det{\mathrm{det}}

\numberwithin{equation}{section}

\usepackage{algorithmic}
\usepackage{algorithm}

\usepackage{epstopdf}
\usepackage[tight,footnotesize]{subfigure}

\begin{document}

\begin{frontmatter}

\title{ICA based on Split Generalized Gaussian}


\author[mymainaddress]{P. Spurek\corref{mycorrespondingauthor}}
\cortext[mycorrespondingauthor]{Corresponding author}
\ead{przemyslaw.spurek@ii.uj.edu.pl}

\author[mysecondaryaddress]{P. Rola}
\ead{przemyslaw.rola@outlook.com}

\author[mymainaddress]{J. Tabor}
\ead{jacek.tabor@ii.uj.edu.pl}

\author[mysecondaryaddress1]{A. Czechowski}
\ead{aleksander.czechowski@dynniq.com}

\address[mymainaddress]{Faculty of Mathematics and Computer Science, 
Jagiellonian University, 
\L ojasiewicza 6, 
30-348 Cracow, 
Poland}
\address[mysecondaryaddress]{Department of Mathematics of the Cracow University of Economics, 
Rakowicka 27,
31-510 Cracow, 
Poland}
\address[mysecondaryaddress1]{Dynniq B.V.,
Basicweg 16,
3821 BR Amersfoort,
The Netherlands}

\begin{abstract}
Independent Component Analysis (ICA) - one of the basic tools in data analysis - aims to find a coordinate system in 
which the components of the data are independent. 
Most popular ICA methods use kurtosis as a metric of non-Gaussianity to maximize, such as FastICA and
JADE. However, their assumption of fourth-order moment (kurtosis) may not always be satisfied in practice.
One of the possible solution is to use third-order moment (skewness)  instead of kurtosis, which was applied in $ICA_{SG}$ and EcoICA. 

In this paper we present a competitive approach to ICA based on the Split Generalized Gaussian distribution (SGGD), which is well adapted to heavy-tailed as well as asymmetric data. 
Consequently, we obtain a method which works better than the classical 
approaches, in both cases: heavy tails and non-symmetric data.
\end{abstract}


\end{frontmatter}

\section{Introduction}

Independent component analysis (ICA) is a popular unsupervised learning method with many applications. ICA has been applied in magnetic resonance \cite{beckmann2004probabilistic}, MRI \cite{beckmann2005tensorial,rodriguez2012noising}, EEG analysis \cite{brunner2007spatial,delorme2007enhanced,zhang2013bayesian},
fault detection \cite{choi2005fault}, financial time series \cite{kiviluoto1998independent} and  seismic recordings \cite{haghighi2008ica}.

In our work we introduce and explore a new approach based on the maximum likelihood estimation framework and the non-symmetric and heavy tailed density distribution -- Split Generalized Gaussian Independent Component Analysis (\ICA). Our work is a generalization of \cite{spurek2017ica}, where $ICA_{SG}$, was constructed. 
The motivation for such modification comes from the observation that it is often profitable to describe data by asymmetric and heavy tailed distributions to obtain a better fit of data.

Before explaining our idea, let us first recall for the convenience of the reader some earlier approaches.
Most ICA methods are based on the maximization of non-Gaussianity. 
\begin{figure}[!h]
\normalsize
\begin{center}
\includegraphics[width=4.5in]{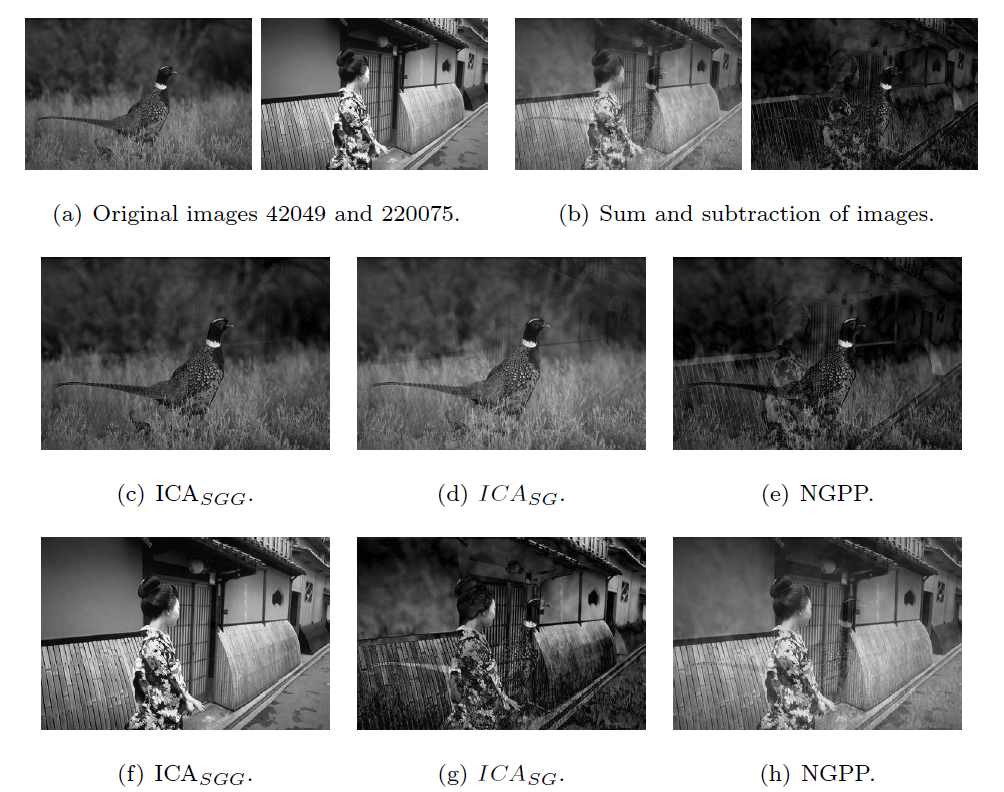}
\end{center}
\caption{Comparison of image separation by our method (\ICA), with $ICA_{SG}$ and NGPP.}
\label{fig:image_ICA_int}
\end{figure}
%
Kurtosis, the classical measure of non-Gaussianity, is used in particular by FastICA \cite{hyvarinen1999fast,helwig2013critique}.
The assumption of kurtosic sources may not always be satisfied in practice.  Typically data sets are bounded, and therefore the credible estimation of tails is not easy. Another problem with these methods, is that they usually assume that the underlying density is symmetric, which is rarely the case.
For weak-kurtosic but skewed sources, such methods could fail \cite{song2016ecoica,spurek2017ica}.
Skewness (the third central moment) is another metric using in ICA. Any symmetric data, in particular the gaussian one, has skewness equal to zero.
One of the most popular ICA methods dedicated for skew data is PearsonICA \cite{karvanen2000pearson,karvanen2002blind}. 

On the other hand, in \cite{spurek2017ica} authors present an approach to ICA based on the maximum likelihood estimation~\cite{pham1997blind}. In such a case we search for the coordinate system optimally fitted to data as well as the marginal
densities  such that the data density factors in the base are the product of marginal densities. 
Authors model skewness using the Split Gaussian distribution, which is well adapted to asymmetric data.  

Unfortunately, all the above approaches work well only for asymmetric and weak-kurtosic source. Our goal is to find a method which is able to work in both situations. 
One of the possible solution is to use a mixture of skewness and kurtosis. In~\cite{blaschke2004cubica,virta2016projection} authors use the projection index which is a combination of third and fourth cumulants. The proposed method gives good results but it is a problem with modeling the proportion between skewness and kurtosis.

In our work we introduce a new approach to ICA in which we approximate the data density by product of 
Split Generalized Gaussian distribution, which allows us to simultaneously model skewness and heavy-tails in data.
Thanks to Theorem \ref{the:min} we reduce the minimization of the maximum likelihood function from five to three parameters. Moreover, in Theorem \ref{main:grad} we give an explicit formula for gradient of the cost function, which allows the use of classical gradient descent method.
Consequently we obtain ICA method which gives essentially better results then classical approaches with similar computational complexity.

We verified \ICA \ in the case of density estimation of images and found the optimal parameters of Logistic, Split Gaussian, Split Generalized Gaussian distributions. In Fig. \ref{fig:MLE} we compared the values of the MLE function. In most of the cases Split Generalized Gaussian distribution fits the data with better precision. 
\begin{figure*}[!h]
\normalsize
\begin{center}
\includegraphics[width=5.0in]{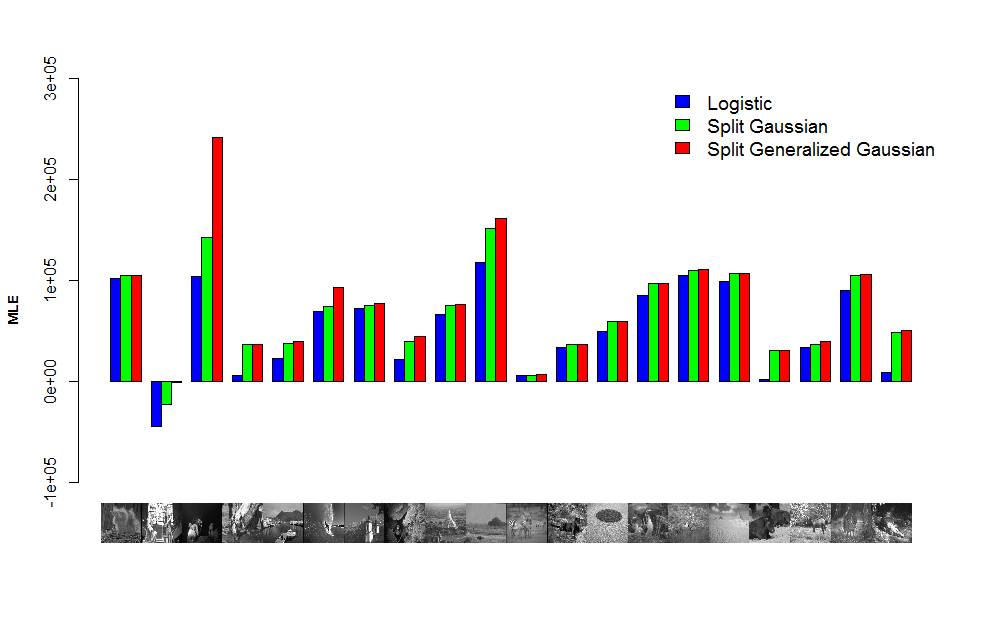} 
\end{center}
\caption{The MLE estimation for image histograms with respect to Logistic, Split Gaussian and Split Generalized Gaussian distributions.}
\label{fig:MLE}
\end{figure*}
%
%

The results of $ICA_{SG}$ \cite{spurek2017ica} (described in our previous article), NGPP \cite{virta2016projection} (which use a combination of third and fourth cumulants) and our method \ICA \ are compared in Fig. \ref{fig:image_ICA_int} for the case of image separation (for more detail comparison we refer to Section \ref{ex}). In the experiment we mixed two images (see Fig. \ref{fig:image_ICA_int}) by adding and subtracting them. Our approach gives essentially better results. In the case of other ICA methods we can see artifacts in the background, which means that the method does not separate signal properly.

This paper is arranged as follows. In the second section, we discuss related works. In the third one, the theoretical background of our approach to ICA is presented. We introduce a cost function which uses the General Split Gaussian distribution and show that it is enough to minimize it respectively to only two parameters: vector $\m \in \R^d$  and $d \times d$ matrix  $W$. We also calculate the gradient of the cost function, which is necessary for the efficient use in the minimization procedure.
The last section describes the numerical experiments. The effects of our algorithm are illustrated on simulated as well as real datasets.

\section{Related works}\label{RW}


Various ICA methods were discussed in the literature \cite{secchi2016hierarchical, hyvarinen2004independent,lee1999independent,cardoso1989source,pham1997blind,comon1994independent} and many practical application were proposed. 
 In signal processing ICA is a computational method for separating a multivariate signal into additive subcomponents and has been applied in magnetic resonance \cite{beckmann2004probabilistic}, MRI \cite{beckmann2005tensorial,rodriguez2012noising}, EEG analysis \cite{brunner2007spatial,delorme2007enhanced,zhang2013bayesian},
fault detection \cite{choi2005fault}, financial time series \cite{kiviluoto1998independent} and  seismic recordings \cite{haghighi2008ica}.
Moreover, it is hard to overestimate the role of ICA in pattern recognition and image analysis; its applications include face recognition \cite{yang2005kernel,dagher2006face}, facial action recognition~\cite{chuang2006recognizing},  image filtering \cite{tsai2006independent}, texture segmentation \cite{jenssen2003independent}, object recognition~\cite{bressan2003using}, image modeling \cite{kim2005iterative}, embedding graphs in pattern-spaces \cite{luo2003spectral,luo2002independent} and feature extraction \cite{lai2014multilinear}. 

The first ICA method was presented by Herault and Jutten around 1983. The authors proposed an iterative real-time algorithm based on a neuro-mimetic architecture~\cite{jutten1991blind}. It is worth mentioning that in their
framework, higher-order statistics were not introduced
explicitly. Giannakis et al. \cite{giannakis1989cumulant} addressed the issue of
identifiability of ICA in 1987 using third-order cumulants. However, the resulting algorithm required an exhaustive search. 

Lacoume and Ruiz \cite{lacoume1992separation} sketched
a mathematical approach to the problem using
higher-order statistics, which can be interpreted as a measure of fitting independent components. 
Cardoso \cite{cardoso1991super,cardoso1999high} focused on the algebraic properties of the fourth-order cumulant (kurtosis), which is still a popular approach~\cite{sharma2006subspace}.

An important measure of fitting independent components is given by negentropy \cite{gaeta1990source}. FastICA \cite{hyvarinen1999fast}, one of the most popular implementations of ICA,  uses this approach.
Negentropy is based on the information-theoretic quantity of (differential) entropy. This concept leads to the mutual information
which is the natural information-theoretic measure of the independence of random variables. Consequently, one can use it as the criterion for finding the ICA transformation \cite{comon1994independent,bell1995information}. It can be shown that minimization of the mutual information is roughly equivalent to maximization of negentropy and it is easier to estimate since we do not need additional parameters.

A somewhat similar approach to ICA is based on the maximum likelihood estimation~\cite{pham1997blind}. It is closely connected to the infomax principle since the likelihood is proportional to the negative of mutual information. 
In recent publications, the maximum likelihood estimation is one of the most popular \cite{hyvarinen2004independent,harroy1996maximum,comon2010handbook,samworth2012independent,zarzoso2006optimal,murillo2004sinusoidal,cardoso2006maximum} approaches to ICA.  In our paper we also use the maximum likelihood framework. 

\section{Maximum likelihood approach to ICA}\label{RW}

Let us now, for the readers convenience, describe how the 
method\footnote{In fact it is one of the possible approaches, as there are many explanations which lead to similar formula.}~works \cite{hyvarinen2000independent}. Suppose that we have a random vector $X$
in $\R^d$ which is generated by the model with the density $F$. Then it is well-known that components of $X$ are independent iff there exist one-dimensional densities $f_1,\ldots,f_d \in \D_\R$, where by $\D_\R$ we denote the set of densities on $\R$, such that
$$
F(\x)=f_1(x_1) \cdot \ldots \cdot f_d(x_d), \for
\x=(x_1,\ldots,x_d) \in \R^d.
$$
Now suppose that the components of $X$ are not independent, but that
we know (or suspect) that there is a basis $A$ (we put $W=A^{-1}$) such that in that base the
components of $X$ become independent. This may be formulated in the form
\begin{equation} \label{eq:gen}
F(\x)=\det(W) \cdot f_1(\w_1^T(\x-\m)) \cdot \ldots \cdot f_d(\w_d^T(\x-\m)) \for x \in \R^d,
\end{equation}
where $\w_i^T(\x-\m)$ is the $i$-th coefficient of $\x-\m$ (the basis is centered in $\m$) in the basis $A$ ($\w_i$ denotes the $i$-th column of $W$).
Observe, that for a fixed family of one-dimensional densities $\F \subset \D_\R$, the set of all densities given by \eqref{eq:gen} for $f_i \in \F$, forms an affine invariant set of densities.

Thus, if we want to find such a basis that components become independent, we need to search for a matrix $W$ and one-dimensional densities such that the approximation 
$$
F(\x) \approx \det(W) \cdot f_1(\w_1^T(\x-\m)) \cdot \ldots \cdot f_d(\w_d^T(\x-\m)), \for \x \in \R^d
$$
is optimal. 

The above approximation can be done using the maximum likelihood estimation, which leads to the flowing formulation of the ICA problem:

\medskip

\noindent{\bf Problem }{\em
Let $X \subset \R^d$ be a data set and $\F \subset \D_\R$ be a set of densities. Find an unmixing matrix $W$, center $\m$, and densities $f_1,\ldots,f_d \in \F$ so that the value 
$$
\frac{1}{|X|}\sum_{i=1}^d \sum_{\x \in X} \ln(f_i(\w_i^T(\x-\m)))+\ln(\det(W)) 
$$
is maximized.
}

It may seem that the most natural choice is Gaussian densities. However, this is not the case as Gaussian densities are affine invariant, and therefore do not ``prefer'' any fixed choice of coordinates\footnote{In fact one can observe that the choice of gaussian densities leads to PCA, if we restrict to the case of orthonormal bases}. In other words we have to choose a family of densities which is distant from Gaussian ones.

In the classical ICA approach it is common to use the super-Gaussian logistic distribution:
$$
f(x; \mu,s) = \frac{e^{\frac{x-\mu}{s}}} {s\left(1+e^{\frac{x-\mu}{s}}\right)^2} =\frac{1}{4s} \operatorname{sech}^2\!\left(\frac{x-\mu}{2s}\right).
$$
The main difference between the gaussian and super-gaussian is the existence of heavy-tails. This can be also viewed as the difference in the fourth moments.

Another approach is based on the Split Gaussian distribution
$$
SN(x;m,\sigma^2,\tau^2) = \left\{ \begin{array}{ll}
c \cdot \exp[-\frac{1}{2\sigma^2}(x-m)^2], & \textrm{where $x\leq m$}\\
c \cdot \exp[-\frac{1}{2\tau^2\sigma^2}(x-m)^2], & \textrm{where $x>m$}\\
\end{array} \right.
$$
where $c=\sqrt{\frac{2}{\pi}}\sigma^{-1}(1+\tau)^{-1}$.
As we see, the split normal distribution arises from merging two opposite halves of two probability density functions of normal distributions in their common mode.
In general, the use of the Split Gaussian distribution allows to fit skew data. 
The main difference between the gaussian and SN can be also viewed as the difference in the third moments.

\begin{figure*}[!t]
\normalsize
\begin{center}
\includegraphics[width=4.5in]{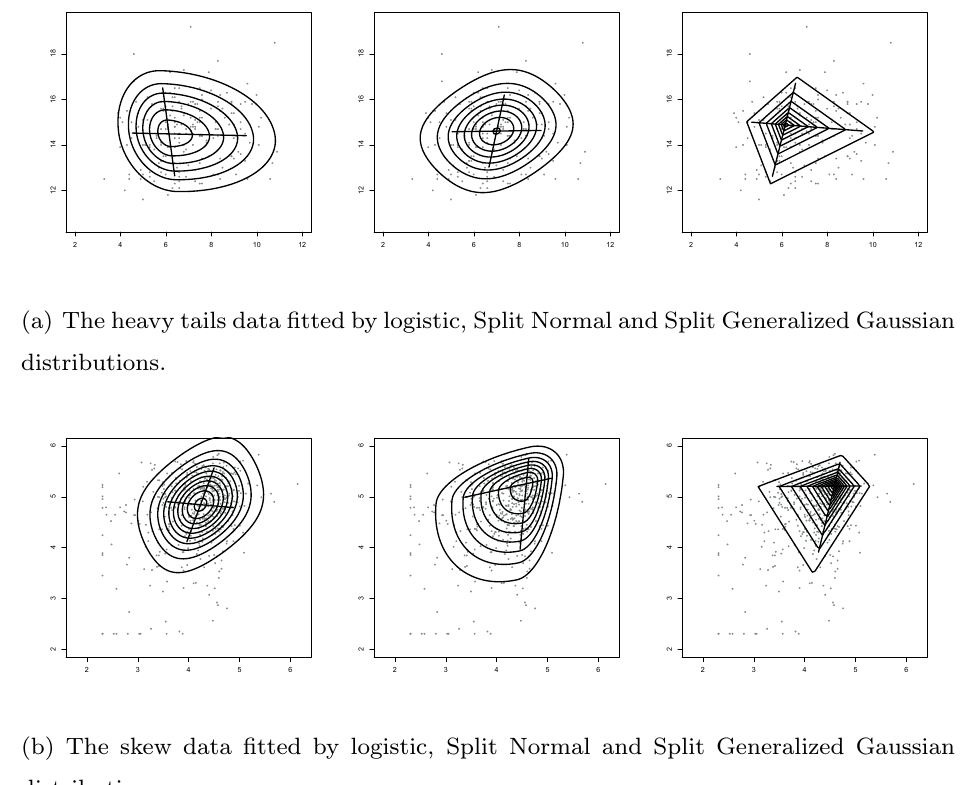} 
\end{center}
\caption{Logistic and Split Normal distributions fitted to data with heavy and skew tails, respectively.}
\label{fig:den_2d}
\end{figure*}


Both of the above choices have advantages and disadvantages. In the first case the model is very sensitive to outliers and the approximation could not give the expected results for asymmetric data. On the other hand, skew model do not fit heavy tails, which are quite common in EEG datasets. 

The idea behind this paper was to choose the model of densities which wouldn't
have the two above disadvantages. So, instead of choosing the family which differs from the Gaussians by the size of tail (fourth moment) or skewness of (third moment), we chose a family which would allow for the estimation of both of the mentioned measures simultaneously -- Split Generalized Gaussian (SGG) distribution \cite{nandi1995extension,tesei1998hos}.

\section{Split Gaussian distribution}\label{SGD}


In this section we present our density model.
Natural directions for extending the normal distribution are to introduce skewness or heavy-tails, and several proposals have indeed emerged, both in the univariate and multivariate case, see \cite{azzalini1985class,azzalini1996multivariate,villani2006multivariate,gibbons1973estimation,nandi1995extension,tesei1998hos,pascal2013parameter}.
One of the most popular approaches for skewness is the Split Normal (SN) distribution\cite{gibbons1973estimation} and for heavy tails is the General Gaussian (GG) distribution \cite{nandi1995extension,tesei1998hos,pascal2013parameter}.

In our paper we use a generalization of the above models, which we call the  Split Generalized Gaussian (SGG) distribution. 
We start from the one-dimensional case. After that, we present a possible generalization of this definition to the multidimensional setting, which corresponds to the formula (\ref{eq:gen}). Contrary to the Split Gaussian distribution, we  skip the assumption of the orthogonality of coordinates (often called principal components), and obtain an ICA model.

\subsection{The one-dimensional case}

Main limitations of normal distribution are its symmetry, and
the capability of controlling the pdf shape only by
measuring the deviation of samples with respect to
the mean without modeling tails.

As it was mentioned, most ICA methods are based on the maximization of non-Gaussianity.
One of the most common and simplest parameters
able to describe deviation from normality is
skewness defined as the third-order central
moment of a stochastic variable.
It was found \cite{tacconi1995new} that the information it can provide is equivalent to that yielded by the combination of two
empirical parameters, the ``left and right variances''.

In order to modify the Gaussian pdf to describe
deviation from symmetry, the left and right variances
were proved to be easier to use than skewness
\cite{azzalini1985class,azzalini1996multivariate,villani2006multivariate}. 
Replacing the variance with two different left and right variances
in Gaussian pdf, gave the asymmetric Split Gaussian model:
$$
SN(x;m,\sigma^2,\tau^2) = \left\{ \begin{array}{ll}
c \cdot \exp[-\frac{1}{2\sigma^2}(x-m)^2], & \textrm{where $x\leq m$}\\
c \cdot \exp[-\frac{1}{2\tau^2\sigma^2}(x-m)^2], & \textrm{where $x>m$}\\
\end{array} \right.
$$
where $c=\sqrt{\frac{2}{\pi}}\sigma^{-1}(1+\tau)^{-1}$.

As we see, the split normal distribution arises from merging two opposite halves of two probability density functions of normal distributions in their common mode.
In general, the use of the Split Gaussian distribution (even in 1D) allows to fit data with better precision (from the likelihood function point of view). 

Another measure of non-Gaussianity in terms of shape is represented by the kurtosis. The parameter is equal to three in the Gaussian case;
the sharpness of the pdf shape is higher (lower) than the corresponding Gaussian function when the parameter is larger (smaller) than three.
A good model for generalized symmetric pdfs has
a variable sharpness. One of the most widely used
symmetric pdf models with a variable sharpness is
the Generalized Gaussian \cite{miller1972detectors,nandi1995extension}

$$
GG(x;m,\alpha,c) = 
\frac{c}{2\alpha \Gamma(1/c)} \exp\left[-\frac{|x-m|^c}{\alpha^c}\right],  
$$
for $m \in \R$ and $\alpha,c\in\R_{+}$ where $\Gamma$ is the standard Gamma function. The parameter $c$, which is theoretical ($c>0$), influences the model sharpness, but cannot be estimated directly from data samples.

The main limitation affecting the generalized Gaussian model is the symmetry.
As the left and right variances were replaced by the
variance in the Gaussian pdf in order to build the
asymmetric Gaussian model, these two parameters
are introduced into the kurtosis-based generalized
Gaussian pdf in a similar way, by transforming it
into the following asymmetric -- Split Generalized Gaussian
model:
$$
SGG(x;m,\sigma_l,\sigma_r,c) = \left\{ \begin{array}{ll}
\frac{c}{(\alpha_l+\alpha_r)\Gamma(1/c)} \exp[-\frac{\vert x-m \vert^c}{\alpha_l^c}] & \textrm{where $x< m$}\\
\frac{c}{(\alpha_l+\alpha_r)\Gamma(1/c)} \exp[-\frac{\vert x-m \vert^c}{\alpha_r^c}] & \textrm{where $x\geq m$}\\
\end{array} \right.
$$
for $m \in \R$ and $\sigma_l,\sigma_r,c\in\R_{+}$.
The relation between $\alpha_l,\alpha_r$ and standard deviations $\sigma_l,\sigma_r$ is
$$
\alpha_i=\sigma_{i}\sqrt{ \frac{\Gamma(1/c)}{\Gamma(3/c)} }, \mbox{ for } i =l,r.
$$

\subsection{Multidimensional Split Gaussian distribution }
A natural generalization of the univariate Generalized Gaussian distribution to the multivariate settings was presented in \cite{pascal2013parameter}.
Roughly speaking, authors assume that a vector $\x \in \R^d$ follows the multivariate Generalized Gaussian distribution, if its principal components are orthogonal and follow the one-dimensional Generalized Gaussian distribution.

In this article we introduce a possible generalization of the Split Generalized Gaussian distribution, but without the assumption of the orthogonality. The construction of the model is similar to the multivariate Split Gaussian distribution presented in \cite{spurek2017ica} for $ICA_{SG}$ method. Thanks to the use of the Split Generalized Gaussian distribution we can model skewness and kurtosis at the same time.

\begin{definition}\label{def:GSN}
A density of the multivariate Split Generalised Gaussian distribution is given by
$$
SGG_{d}(\x;\m,W,\sigma_l,\sigma_r,c)=|\det(W)| \prod_{j=1}^{d} SGG(\w_j^T(\x-\m);0,\sigma_{lj},\sigma_{rj},c),
$$
where $\w_{j}$ is the $j$-th column of non-singular matrix $W$, $\m = (m_1, \ldots, m_d)^T$, $\sigma_l = (\sigma_{l1},\ldots,\sigma_{ld})$, $\sigma_r = (\sigma_{r1},\ldots,\sigma_{rd})$ and $c$ is a constant.\end{definition}

Our model is a natural generalization of the multivariate Generalized Gaussian distribution proposed in \cite{villani2006multivariate} and the multivariate Split Gaussian distribution described in \cite{spurek2017ica}.

The above model is flexible, and allows to fit data with greater precision.  
In the next section we discuss how to estimate optimal parameters in our model. 

\section{Maximum likelihood estimation}

In the previous section we introduced the SGG distribution. 
Now we show how to use the likelihood estimation in our setting. As it was mentioned, we have to maximize the likelihood function with respect to five parameters. In the case of the Split Generalized Gaussian distribution (contrary to the classical Gaussian one) we do not have explicit formulas and consequently we heave to solve the optimization problem.

In the first subsection, we reduce our problem to the simpler one by introducing an auxiliary function~${l}$. Minimization of~${l}$~is equivalent to maximization of the likelihood function.
In the second subsection we present how to minimize our function by using the gradient method.

\subsection{Optimization problem}

The density of the SGG distribution depends on five parameters $\m \in \R^d$, $W \in \M(\R^d)$, $\sigma_l \in \R^d$, $\sigma_r \in \R^d$ and $c \in \R$. 
We can find them by minimizing the simpler function, which depends on only  $m \in \R^d$, $W \in \M(\R^d)$ and $c\in\R$. Other parameters are given by the explicit formulas.    

\begin{theorem}\label{the:min}
Let $\x_1,\ldots,\x_n$ be given.  
Then the likelihood maximized w.r.t. $\sigma_l$ and $\sigma_r$ is
\begin{equation}\label{eq:1}
 \hat{L}(X;\m,W,c) = \bigg( \frac{\kappa n}{c e} \bigg)^{\frac{dn}{c}} \Big( |\det(W)|^{-\frac{c}{c+1}} \prod\limits_{j=1}^{d} g_j(\m,W) \Big)^{-\frac{n(c+1)}{c}}  
\end{equation}
where $\kappa = \left[ \frac{1}{c} \Gamma(\frac{1}{c}) \right]^{-c}$ and
$$
\begin{array}{c}
{g}_{j}(\m,W,c) = {s}_{1j}^{\frac{1}{c+1}} + {s}_{2j}^{\frac{1}{c+1}},
\\[1ex]
{s}_{1j}= \! \sum\limits_{i \in I_j}\vert \w_{j}^T (\x_i-\m) \vert^c,  {I}_j=\{ i = 1,\ldots,n \colon \w_{j}^T (\x_i-\m) \leq 0 \},
\\[1ex]
{s}_{2j}= \! \sum\limits_{i \in I_j^c}\vert \w_{j}^T (\x_i-\m) \vert^c, {I}_j^{'}=\{ i = 1,\ldots,n \colon  \w_{j}^T (\x_i-\m) > 0 \},
\end{array}
$$
and the maximum likelihood estimators of $\alpha_{lj}$, $\alpha_{rj}$ are
$$
\hat{\alpha}_{lj} = \hat{\sigma}_{lj} \sqrt{\frac{\Gamma(\frac{1}{c})}{\Gamma(\frac{3}{c})}} \qquad \text{and} \qquad \hat{\alpha}_{rj} = \hat{\sigma}_{rj} \sqrt{\frac{\Gamma(\frac{1}{c})}{\Gamma(\frac{3}{c})}}
$$
where the estimators of $\sigma_l$ and $\sigma_r$ are given by
$$
\hat{\sigma}_{lj}^c(\m,W) = 
\tfrac{c}{n} \beta^{\frac{c}{2}} s_{1j}^{\frac{c}{c+1}} g_{j}(\m,W), \qquad \hat{\tau}_{j}(\m,W) = \bigg( \frac{s_{2j}}{s_{1j}} \bigg)^{\frac{1}{c+1}}
$$
and
$$
\hat{\sigma}_{rj}^c(\m,W) = \hat{\sigma}_{lj}^c(\m,W) \cdot \hat{\tau}^c_{j}(\m,W) =
\tfrac{c}{n} \beta^{\frac{c}{2}}  s_{2j}^{\frac{c}{c+1}} g_{j}(\m,W),
$$
where $\beta=\frac{\Gamma(\frac{3}{c})}{\Gamma(\frac{1}{c})}$.

\end{theorem}
\begin{proof}
See Appendix \ref{App:A}.
\end{proof}

Thanks to the above theorem, instead of looking for the maximum of the likelihood function, it is enough to obtain the maximum of the simpler function~(\ref{eq:1}) which depends on three parameters $\m \in \R^d$, $W \in \M(\R^d)$ and $c \in \R$.

%
%
%
%
%

\subsection{Gradient}

One of the possible methods of optimization is the gradient method. For simplicity we calculate gradient of the log-likelihood function.

In the first step we introduce the following function
\begin{equation}\label{equ:ll}
{l}(X;\m,W,c) = |\det(W)|^{-\frac{c}{c+1}} \prod\limits_{j=1}^{d} g_j(\m,W),
\end{equation}
where $\w_{j}$ stands for the $j$-th column of matrix $W$. Let us notice that
\begin{equation}
\ln \hat{L}(X;\m,W,c) = \frac{dn}{c} \ln \left( \frac{\kappa n}{ce} \right) - \frac{n(c+1)}{c} \ln {l}(X;\m,W,c)
\end{equation}

We calculate a gradient of $l$ and then we show the final result.

\begin{theorem}\label{ther:grad}
Let $X \subset \R^d$, $\m = (\m_1, \ldots, \m_d)^T \in \R^d$, $W = (\w_{ij})_{1 \leq i,j \leq d}$ non-singular be given. 
Then
$\nabla_{\m}  \ln {l}(X;\m,W,c) = \left(  \frac{\partial \ln {l}(X;\m,W,c)}{\partial \m_1}, \ldots, \frac{\partial \ln {l}(X;\m,W,c)}{\partial \m_d} \right)^T$,
where $\frac{\partial \ln {l}(X;\m,W)}{\partial \m_k} =$
$$
\begin{array}{l}
\frac{c}{c+1} \sum\limits_{j=1}^d \frac{1}{{s}_{1j}^{\frac{1}{c+1}} + {s}_{2j}^{\frac{1}{c+1}}} \bigg(
{s}_{1j}^{- \frac{c}{c+1}} \sum\limits_{i \in {I}_j} \vert \omega^T_j (\x_i - \m) \vert^{c-1} \omega_{jk} - 
{s}_{2j}^{- \frac{c}{c+1}} \sum\limits_{i \in I_j^{'}} \vert \omega^T_j (\x_i - \m) \vert^{c-1} \omega_{jk}
\bigg).
\end{array}
$$
Moreover,
$
\nabla_{W} \ln {l}(X;\m,W,c) = \left[ \frac{\partial \ln l(X;\m,W,c)}{\partial \w_{pk}}  \right]_{1 \leq p,k \leq d},
$
where
$$
\begin{array}{l}
\frac{\partial \ln l(X;\m,W,c)}{\partial \w_{pk}}  = -\frac{c}{c+1}  (\w^{-1})^T_{pk} + \frac{1}{{s}_{1p}^{\frac{1}{c+1}} +{s}_{2p}^{\frac{1}{c+1}}}
 \bigg(
- \frac{c}{c+1} {s}_{1p}^{-\frac{c}{c+1}} \sum\limits_{ i \in {I}_p} \vert \w^T_p (\x_i - \m) \vert^{c-1} (\x_{ik} - \m_k) \\[6pt]
+ \frac{c}{c+1} {s}_{2p}^{-\frac{c}{c+1}} \sum\limits_{ i \in {I}_p^{'}} \vert \w^T_p (\x_i - \m) \vert^{c-1} (\x_{ik} - \m_k) \bigg).
\end{array}
$$
and
$$
\begin{array}{c}
{s}_{1j}= \! \sum\limits_{i \in I_j}\vert \w_{j}^T (-\x_i+\m) \vert^c, \qquad {I}_j=\{ i = 1,\ldots,n \colon \w_{j}^T (\x_i-\m) \leq 0 \},
\\[1ex]
{s}_{2j}= \! \sum\limits_{i \in I_j^{'} }\vert \w_{j}^T (-\x_i+\m) \vert^c, \qquad {I}_j^{'}=\{ i = 1,\ldots,n \colon  \w_{j}^T (\x_i-\m) > 0 \}.
\end{array}
$$
Finally
$$
\begin{array}{c}
\frac{\partial \ln l(X;\m,W,c)}{\partial c} = - \frac{1}{(c+1)^2} \ln |\det(W)| + \\[6pt]
\sum\limits_{j=1}^d \frac{1}{ {s}_{1j}^{\frac{1}{c+1}} + {s}_{2j}^{\frac{1}{c+1}} } 
\bigg(
\frac{1}{c+1} {s}_{1j}^{-\frac{c}{c+1}} \frac{\partial {s}_{1j} }{\partial c} - \frac{{s}_{1j}^{\frac{1}{c+1}}}{(c+1)^2} \ln {s}_{1j} +
\frac{1}{c+1} {s}_{2j}^{-\frac{c}{c+1}} \frac{\partial {s}_{2j} }{\partial c} - \frac{{s}_{2j}^{\frac{1}{c+1}}}{(c+1)^2} \ln {s}_{2j}
\bigg)
\end{array}
$$
where
$$
\begin{array}{c}
\frac{\partial {s}_{1j} }{\partial c} = \! \sum\limits_{i \in I_j} \vert \w_{j}^T (\x_i-\m) \vert^c \ln \vert \w_{j}^T (\x_i-\m) \vert,
\\[1ex]
\frac{\partial {s}_{2j} }{\partial c} = \! \sum\limits_{i \in I_j^{'}} \vert \w_{j}^T (\x_i-\m) \vert^c \ln \vert \w_{j}^T (\x_i-\m) \vert.
\end{array}
$$
\end{theorem}
\begin{proof}
See Appendix \ref{App:B}.
\end{proof}

Now we are ready to calculate the gradient of the log-likelihood function.

\begin{theorem}\label{main:grad}
Let $X \subset \R^d$, $c \in \R$, $\m = (\m_1, \ldots, \m_d)^T \in \R^d$, $W = (\w_{ij})_{1 \leq i,j \leq d}$ non-singular be given. 
Then
\begin{equation}
\nabla_{\m}  \ln \hat{L}(X;\m,W,c) = \left(  \frac{\partial \ln \hat{L}(X;\m,W,c)}{\partial \m_1}, \ldots, \frac{\partial \ln \hat{L}(X;\m,W,c)}{\partial \m_d} \right)^T,
\end{equation}
where
\begin{equation}
\frac{\partial \ln \hat{L}(X;\m,W,c)}{\partial \m_k} = - \frac{n(c+1)}{c} \frac{\partial \ln {l}(X;\m,W,c)}{\partial \m_k}.
\end{equation}
Moreover,
$
\nabla_{W} \ln \hat{L}(X;\m,W,c) = \left[ \frac{\partial \ln \hat{L}(X;\m,W,c)}{\partial \w_{pk}}  \right]_{1 \leq p,k \leq d},
$
where
\begin{equation}
\frac{\partial \ln \hat{L}(X;\m,W,c)}{\partial \w_{pk}} = - \frac{n(c+1)}{c} \frac{\partial \ln {l}(X;\m,W,c)}{\partial \w_{pk}}
\end{equation}
Finally $\frac{\partial \ln \hat{L}(X;\m,W,c)}{\partial c} = $
\begin{equation}
\frac{dn}{c^2} \bigg[ \ln \left( \frac{ce}{n} \right) - 1 + c + \psi \left( \frac{1}{c} \right) \bigg] + \frac{n}{c^2} \ln {l}(X;\m,W,c) - \frac{n(c+1)}{c} \frac{\partial \ln {l}(X;\m,W,c)}{\partial c}
\end{equation}
and $\psi$ is the so-called digamma function, i.e. $\psi (x) = \frac{\Gamma^{'}(x)}{\Gamma(x)}$.
\end{theorem}

\begin{proof}\ref{main:grad}
Recall that

$$\ln \hat{L} = \frac{dn}{c} \ln \left( \frac{\kappa n}{ce} \right) - \frac{n(c+1)}{c} \ln l$$
Then
$$
\begin{array}{l}
\frac{\partial \ln \hat{L}}{\partial c} = \frac{\partial A}{\partial c} + \frac{n}{c^2} \ln {l} - \frac{n(c+1)}{c} \frac{\partial \ln {l}}{\partial c}
\end{array}
$$
where $A = \frac{dn}{c} \ln \left( \frac{\kappa n}{ce} \right)$. Let us calculate $\frac{\partial A}{\partial c}$. Notice that
$$
\begin{array}{l}
A = \frac{dn}{c} \ln \left[ \frac{n}{ce} \left( \frac{1}{c} \Gamma(\frac{1}{c}) \right)^{-c} \right] = \frac{dn}{c} \ln \left( \frac{n}{ce} \right) - dn \ln \left( \frac{1}{c} \Gamma(\frac{1}{c}) \right) = \\
\frac{dn}{c} \ln \left( \frac{n}{e} \right) - \frac{dn}{c} \ln c - dn \ln \Gamma(\frac{1}{c}) + dn \ln c.
\end{array}
$$
Then
$$
\begin{array}{l}
\frac{\partial A}{\partial c} = -\frac{dn}{c^2} \ln \left( \frac{n}{e} \right) + \frac{dn}{c^2} \ln c - \frac{dn}{c^2} - \frac{dn}{\Gamma(\frac{1}{c})} \Gamma^{'}(\frac{1}{c}) \cdot (-\frac{1}{c^2}) + \frac{dn}{c} =\\
\frac{dn}{c^2} \left[ \ln c -1 + c - \ln (\frac{n}{e}) + \frac{\Gamma^{'}(\frac{1}{c})}{\Gamma(\frac{1}{c})} \right].
\end{array}
$$

\end{proof}

Thanks to the above Theorem we can use gradient descent, a first-order optimization algorithm. To find a local maximum of the cost function using gradient descent, one takes steps proportional to the gradient of the function at the current point. 

\begin{figure*}[t!]
\normalsize
\begin{center}
\includegraphics[width=4.5in]{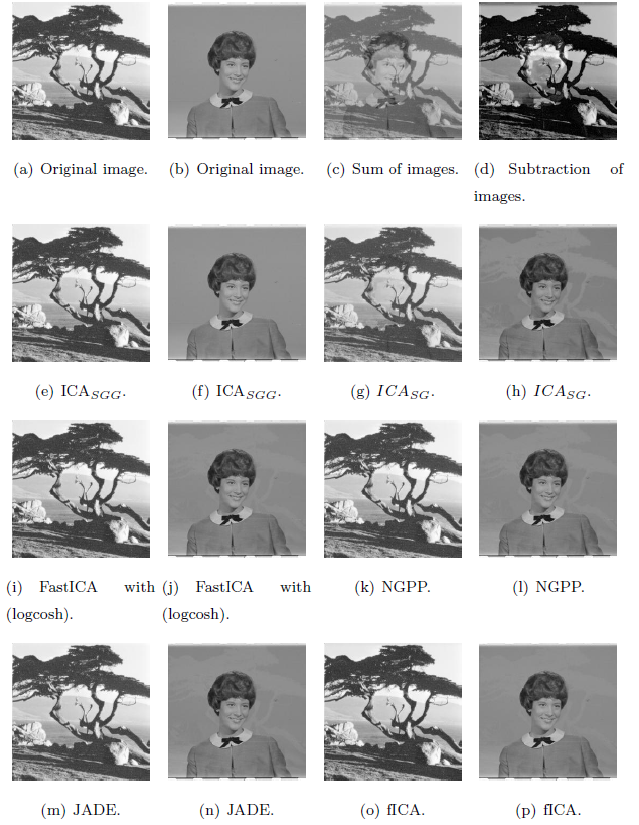} 
\end{center}
\caption{Results of image separation with the use of various ICA algorithms.}
\label{fig:image_ICA_1}
\end{figure*}

\section{Experiments and analysis}\label{ex}

To compare our method to classical ones we use 
Tucker's congruence coefficient \cite{lorenzo2006tucker}
(uncentered correlation) defined by
$$
Cr(\s_i, \bar \s_i) = \frac{ \sum_{j=1}^d s^j_i s^j_i}{ \sqrt{\sum_{j=1}^d s^j_i \sum_{j=1}^d \bar s^j_i } }.
$$
Its values range between $-1$ and $+1$. It can be used to study the similarity of extracted factors across different samples. Generally, a congruence coefficient of $0.9$ indicates a high degree of factor similarity, while a coefficient of $0.95$ or higher indicates that the factors are virtually identical.

We can also verify the quality of recomputing mixing matrix. 
The Amari-Cichocki-Yang (ACY) error is an asymmetric measure of dissimilarity between two
nonsingular square matrices. The ACY error is invariant to permutation and rescaling of the
columns, ranges between $0$ and $N-1$, and equals $0$ if and only if matrices are
identical up to column permutations and rescaling.

The ACY error is defined as
$$
ACY(A_1,A_2)= \frac{  \sum \limits_{i=1}^n \left( \frac{ \sum_{j=1}^d |b_{ij}| }{ \max_{j} |b_{ij}| }  -1 \right)
+ \sum \limits_{j=1}^n \left( \frac{ \sum_{i=1}^d |b_{ij}| }{\max_{i} |b_{ij}| }  -1 \right) }{2n},
$$
where $b_{ij} = (A_1^{-1}A_2)_{ij}$.


We evaluate our method in the context of images, sound 
and EEG data.  
For comparison we use R packages {\tt ica} \cite{ica}, {\tt PearsonICA} \cite{pearsonica}, {\tt ProDenICA} \cite{prodenica}, {\tt tsBSS} \cite{tsBSS}, {\tt fICA} \cite{fICA}, {\tt ICtest} \cite{ICtest}.
The most popular method used in practice is FastICA \cite{hyvarinen1999fast,helwig2013critique} algorithm, which uses negentropy. In this context we can use three different functions to estimate neg-entropy:
logcosh, exp and kurtosis.
We also compare our method with algorithm using Information-Maximization (Infomax) approach \cite{bell1995information}. Similarly to FastICA we consider three possible non-linear functions: hyperbolic tangent, logistic and extended Infomax.
We also consider algorithm which uses Joint Approximate Diagonalization of Eigenmatrices (JADE) proposed by Cardoso and Souloumiac's \cite{cardoso1993blind,cardoso1993blind,helwig2013critique}.

One of the most popular ICA methods dedicated for skew data is PearsonICA \cite{karvanen2000pearson,karvanen2002blind}, which minimizes mutual information using a Pearson \cite{stuart1968advanced} system-based parametric model. Another model we consider is ProDenICA \cite{bach2002kernel,hastie2009elements}, which is based not on a
single nonlinear function, but on an entire function space of candidate nonlinearities. In particular, the method works with the functions in a reproducing kernel Hilbert space, and make use of the “kernel trick” to search over this space efficiently. 
We also compare our method with  NGPP ~\cite{virta2016projection}, which uses the projection index which is a combination of third and fourth cumulants.

\begin{figure*}[!t]
\normalsize
\begin{center}
\subfigure[Dependence of the number of data set instances.] {\label{fig:time_1}
\includegraphics[width=2.1in]{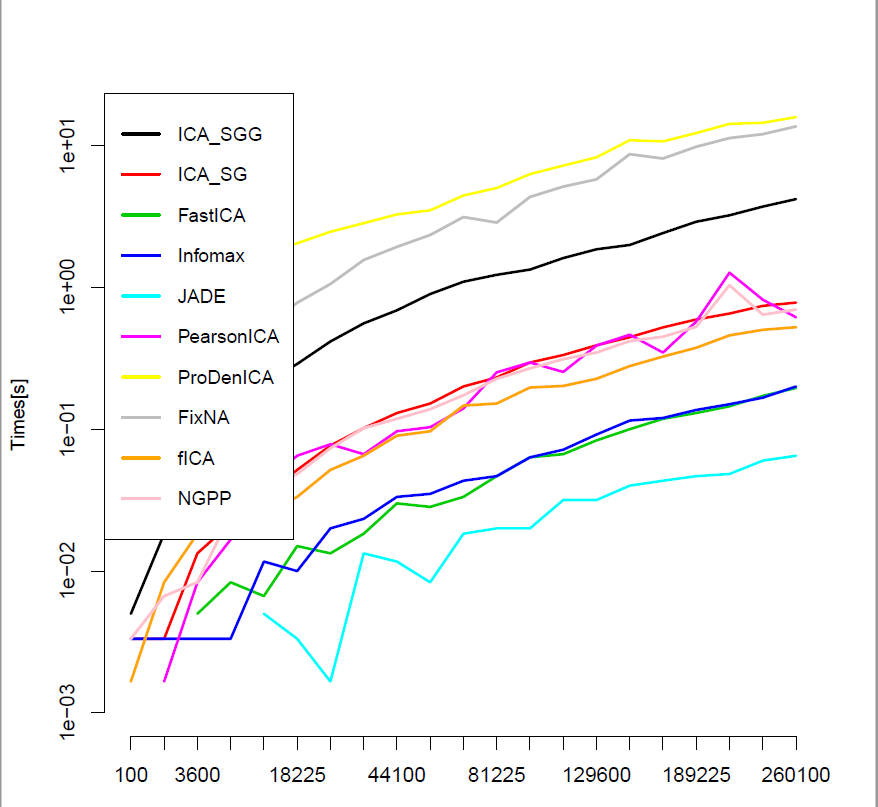} 
} \quad
\subfigure[Dependence of the dimension of data.] {\label{fig:time_2}
\includegraphics[width=2.1in]{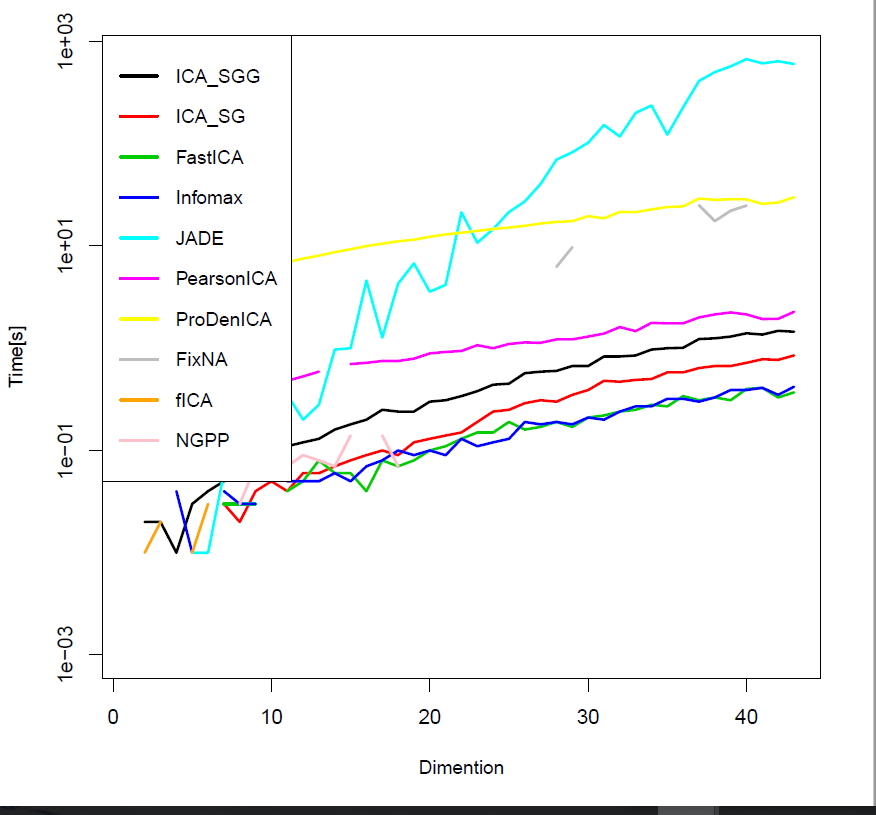}
}
\end{center}
\caption{Comparison of computational efficiency between \ICA{} and classical ICA methods (Time axis is given in the logarithmic scale).}
\label{fig:time}
\end{figure*}

\subsection{Computational efficiency}

First, we verify the computational times of \ICA{} and alternative ICA algorithms. We examine the influence on the number of data set instances and dimension of data. 

We consider the classical image separation problem, where two images are mixed together. We use ten mixed examples and present mean evaluation times. To vary the size of data, images are scaled to different sizes, and running times of the algorithms are reported in each case. One can observe in Figure~\ref{fig:time_1} that \ICA{} is a little bit slower than NGPP but gives comparable results.
 
To examine the influence of data dimension on the evaluation time we also take into account the classical image separation problem, but we change the number of components from 2 to 40. \ICA{} has similar complexity as state of the art method, see Figure~\ref{fig:time_2}. FastICA, Infomax and JADE are the most effective, but do not solve the problem of image separation sufficiently well, see Fig. \ref{fig:rank_img}. On the other hand, the ProDenICA and NGPP which gives comparable result to \ICA{}, have comparable computational time.

%
\subsection{Separation of images}

One of the most popular application of ICA is the separation of images. 
In our experiments we use three hundred images from: the USC-SIPI Image Database (of size $256 \times 256$ pixels and $512 \times 512$ pixels), and from Berkeley Segmentation Dataset of size $482 \times 321$. We make random pairs of above images and use them as a source signal, combined by the mixing matrix $A = \begin{bmatrix} 1 & 1  \\ 1 & -1  \end{bmatrix} $. From practical point of view, we simply obtain two new images by adding and dividing sources pictures. 
Our goal is to reconstruct original images by using only the knowledge about mixed ones. 
As a summary from the experiment, in Fig.~\ref{fig:rank_img} we present a boxplots of ranks obtained by the methods.

\begin{figure*}[!t]
\normalsize
\begin{center}
\subfigure[Boxplots of ranks of ICA methods obtained by using the Tucker's congruence coefficient measure in the separation of images.] {\label{fig:rank_img_1}
	\includegraphics[width=2.1in]{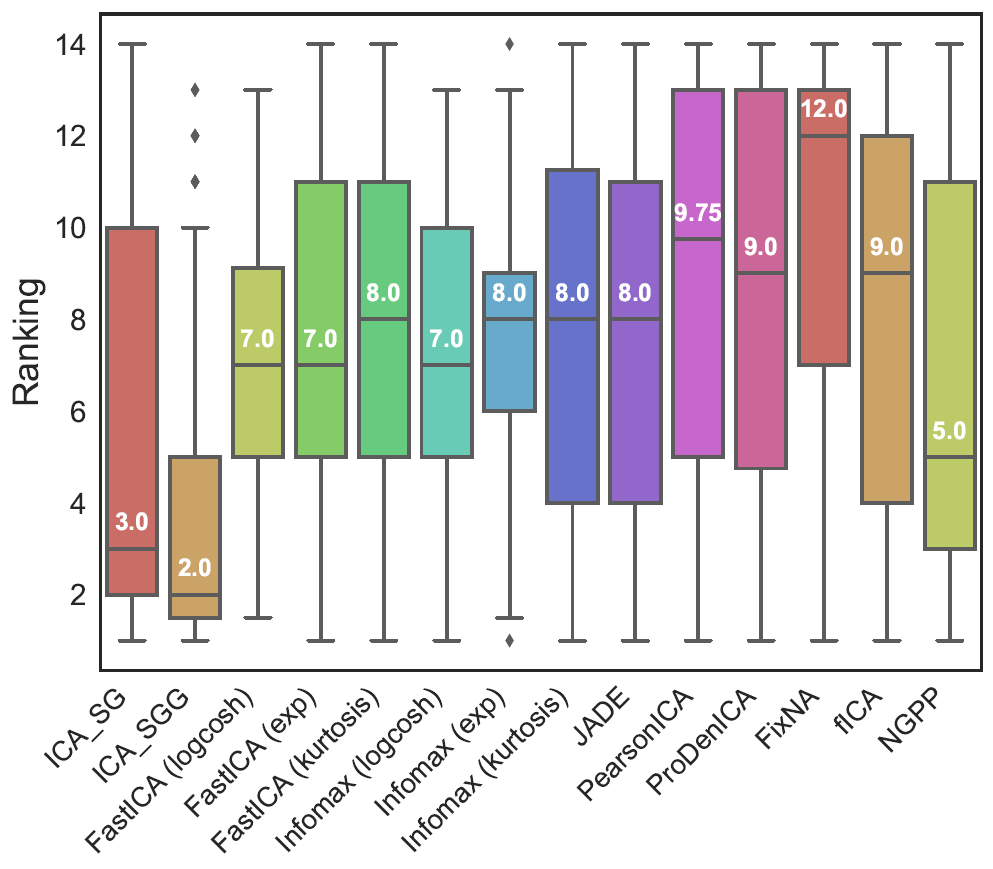} 
}	
\quad 
\subfigure[Boxplots of ranks of ICA methods obtained by using Amari-Cichocki-Yang measures in the separation of images.] {\label{fig:rank_img_2}
	\includegraphics[width=2.1in]{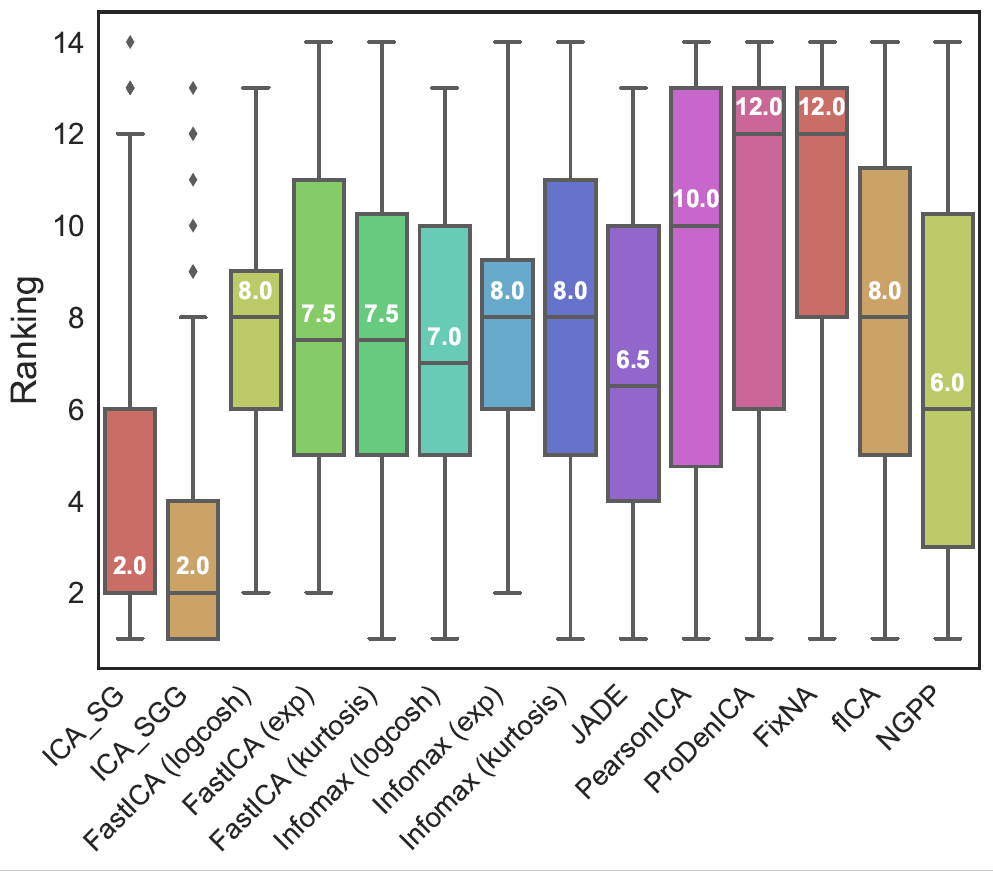} 
}
\end{center}
\caption{Boxplots of ranks of ICA methods obtained by using the Tucker's congruence coefficient and Amari-Cichocki-Yang measures in the separation of images.}
\label{fig:rank_img}
\end{figure*}

In the case of the Tucker's congruence coefficient measure and Amari-Cichocki-Yang error almost in all situation we obtain better results. The \ICA{} method essentially better recovers original signals. In Fig.~\ref{fig:image_ICA_1} we can see that \ICA{} almost perfectly recovers source signal.

\subsection{Cocktail-party problem}

In this subsection we compare our method with classical ones in the case of cocktail-party problem. 
Imagine that you are in a room where two people are speaking simultaneously. You have two microphones, which you hold in different locations. The microphones give you two recorded time signals, which we could interpret as mixed signal $\x$. Each of these recorded signals is a weighted sum of the speech signals emitted by the two speakers, which we denote by $\s$. 
The cocktail-party problem is to estimate the two original speech signals. 

In our experiments we use signal obtained by mixing synthetic sources\footnote{We use signals from \url{http://research.ics.aalto.fi/ica/cocktail/cocktail_en.cgi}.} (similar as before we use mixing matrix $A = \begin{bmatrix} 1 & 1  \\ 1 & -1  \end{bmatrix} $). 
As a summary from the experiment, in Fig.~\ref{fig:rank_sound} we present a boxplots of ranks obtained by the methods.
In the case of cocktail-party problem our method recovers  
sources signal better or comparable to the classical methods. 

\begin{figure*}[!t]
\normalsize
\begin{center}
\subfigure[Boxplots of ranks of ICA methods obtained by using the Tucker's congruence coefficient measure in the case of Cocktail-party problem.] {\label{fig:rank_sound_1}
	\includegraphics[width=2.1in]{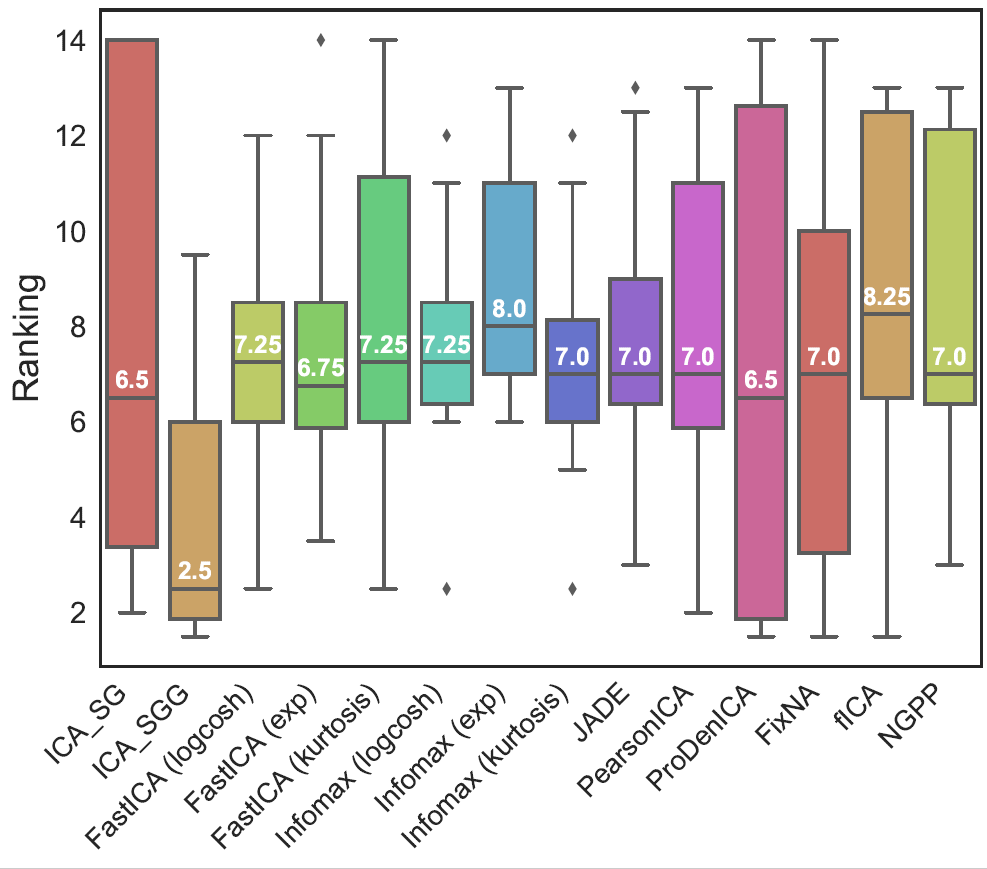} 
}	
\quad 
\subfigure[Boxplots of ranks of ICA methods obtained by using Amari-Cichocki-Yang measures in the case of Cocktail-party problem.] {\label{fig:rank_sound_2}
	\includegraphics[width=2.1in]{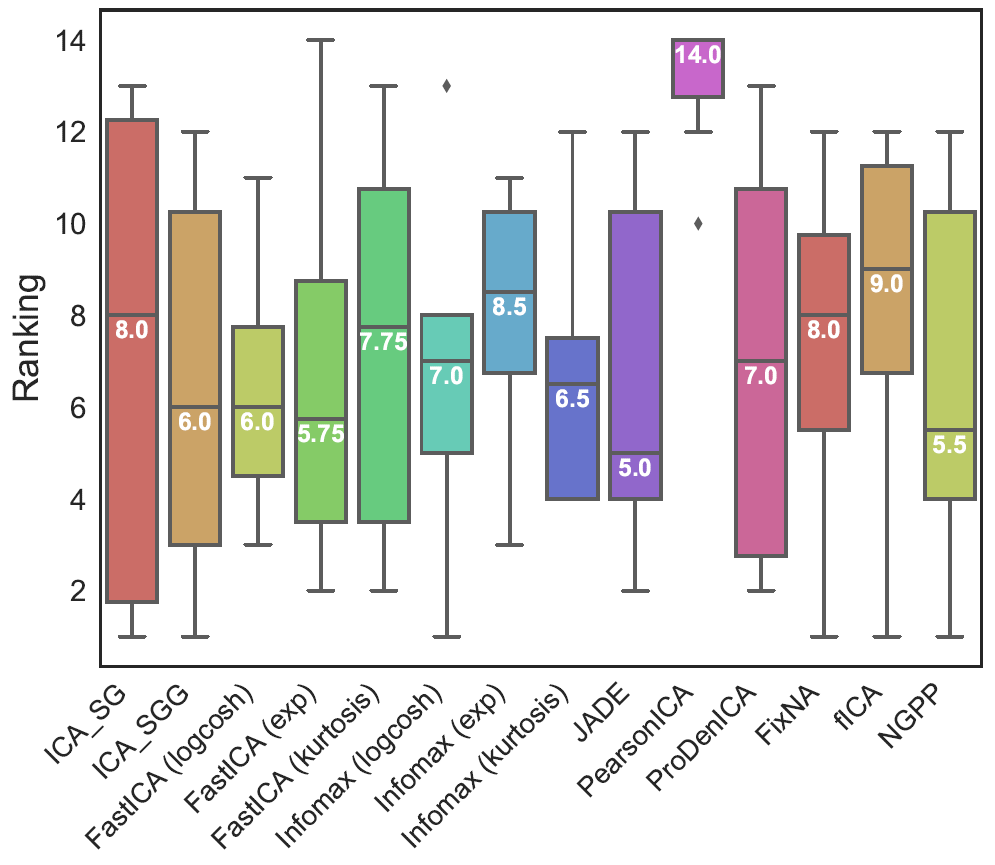} 
}
\end{center}
\caption{Boxplots of ranks of ICA methods obtained by using the Tucker's congruence coefficient and Amari-Cichocki-Yang measures in the case of Cocktail-party problem.}
\label{fig:rank_sound}
\end{figure*}

%
%

\subsection{EEG}

At the end of this section we present how our method works in the case of EEG signals. In this context, ICA is applied to many different task like
eye movements, blinks, muscle, heart and line noise e.t.c.. 
In this experiment we focus on eye movement and blink artifacts. 
Our goal here is to demonstrate that our method is capable of
finding artifacts in real EEG data. However,
we emphasize that it does not provide a complete solution
to any of these practical problems. Such a solution usually
entails a significant amount of domain-specific knowledge
and engineering. Nevertheless, from these preliminary
results with EEG data, we believe that
the method presented in this paper provides a reasonable
solution for signal separation, which is simple and
effective enough to be easily customized for a broad range
of practical problems.

For the EEG analysis, the rows of the input matrix $\x$ are the EEG signals recorded at
different electrodes, the rows of the output data matrix $\s = W\x$ are time courses of
activation of the ICA components, and the columns of the inverse matrix $W$ give
the projection strengths of the respective components onto scalp sensors. 

One EEG data set used in the analysis was collected from 40 scalp electrodes (see Fig. \ref{fig:EEG_1}). The second and the third ones are located very near to the eye and can be understood as a base (we can use them for removing eye blinking artifacts). In Fig. \ref{fig:EEG_2} we present signals obtained by \ICA.  The scale of this figure is large but we can find the data which have spikes exactly in the same place as the two base signals (see Fig. \ref{fig:EEG_3}). After removing selected signal and going back to the original situation we obtain signal (see Fig. \ref{fig:EEG_4}) without eye blinking artifacts (compare Fig. \ref{fig:EEG_1} with  Fig. \ref{fig:EEG_4}).

\begin{figure*}[!t]
\normalsize
\begin{center}
\subfigure[Original signal from EEG.] {\label{fig:EEG_1}
\includegraphics[width=2.1in]{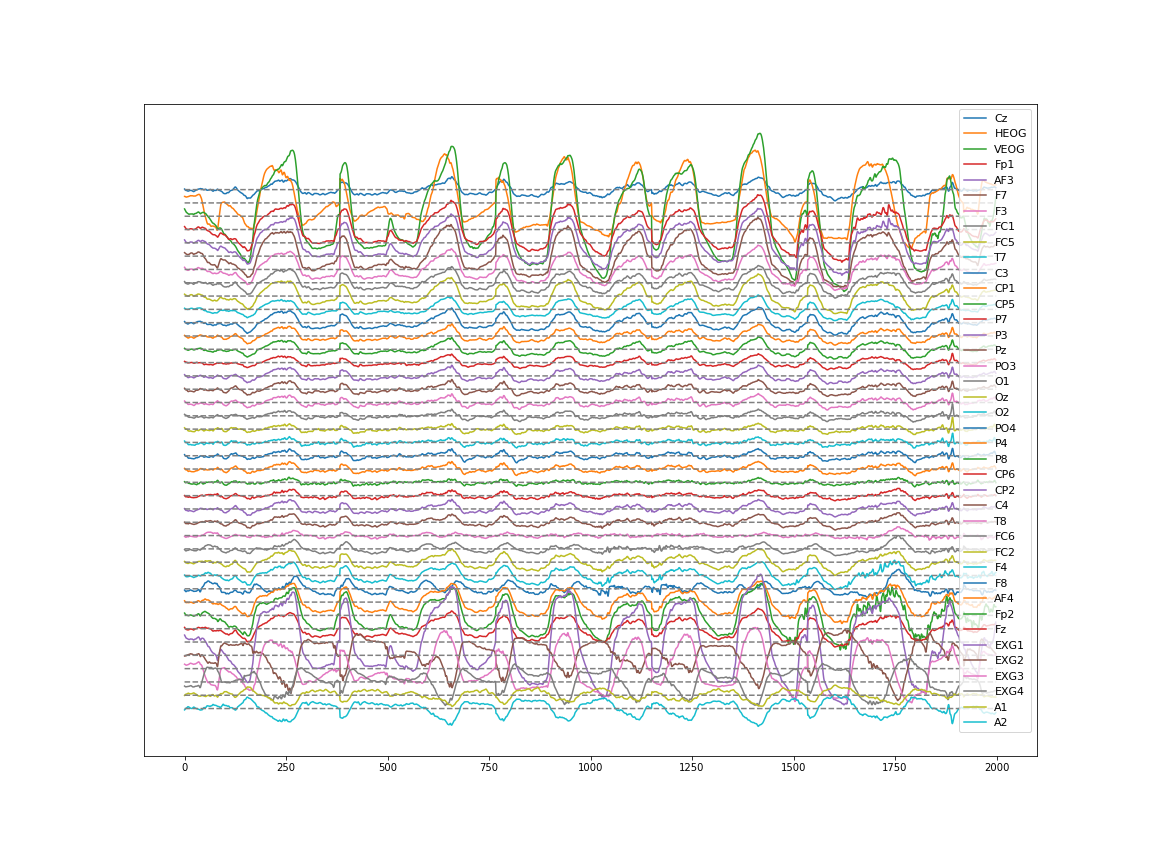} 
}
\subfigure[Sources signals obtained by \ICA.] {\label{fig:EEG_2}
\includegraphics[width=2.1in]{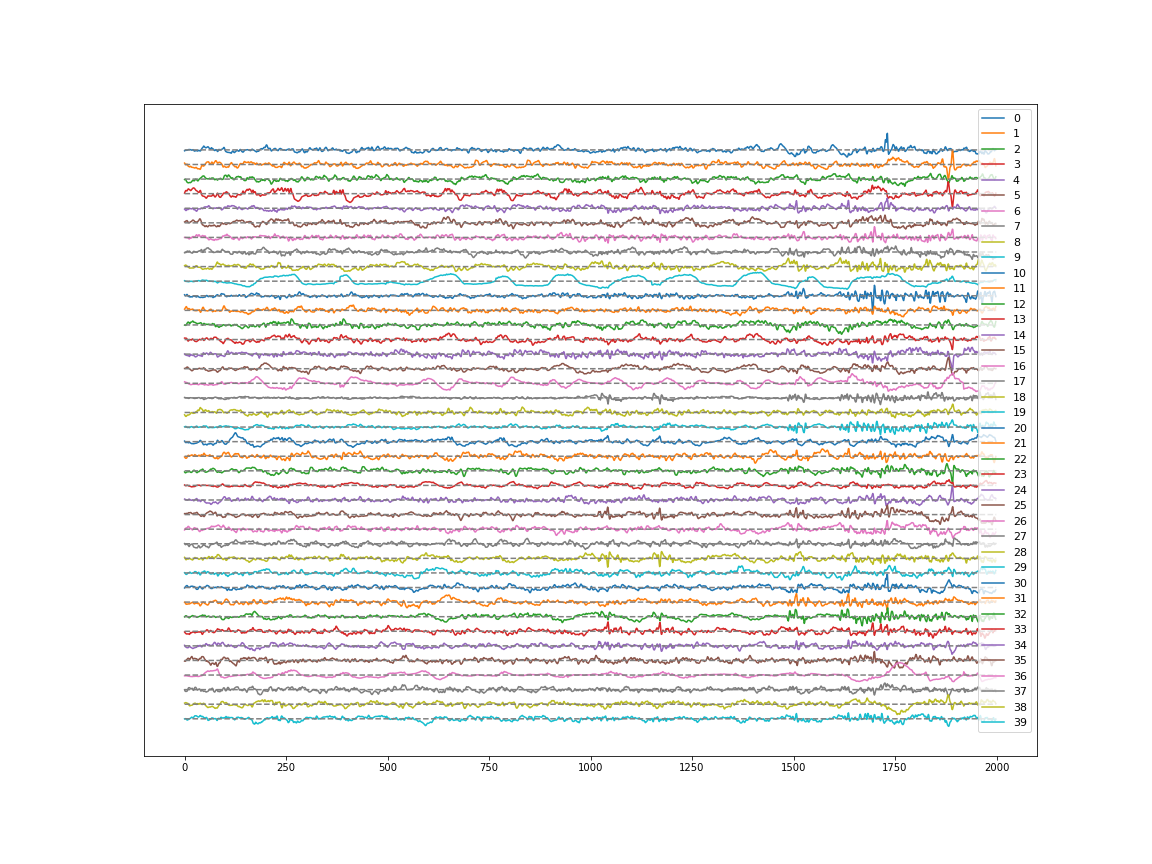}
}
\subfigure[Three components  9, 20, 36.] {\label{fig:EEG_3}
\includegraphics[width=2.1in]{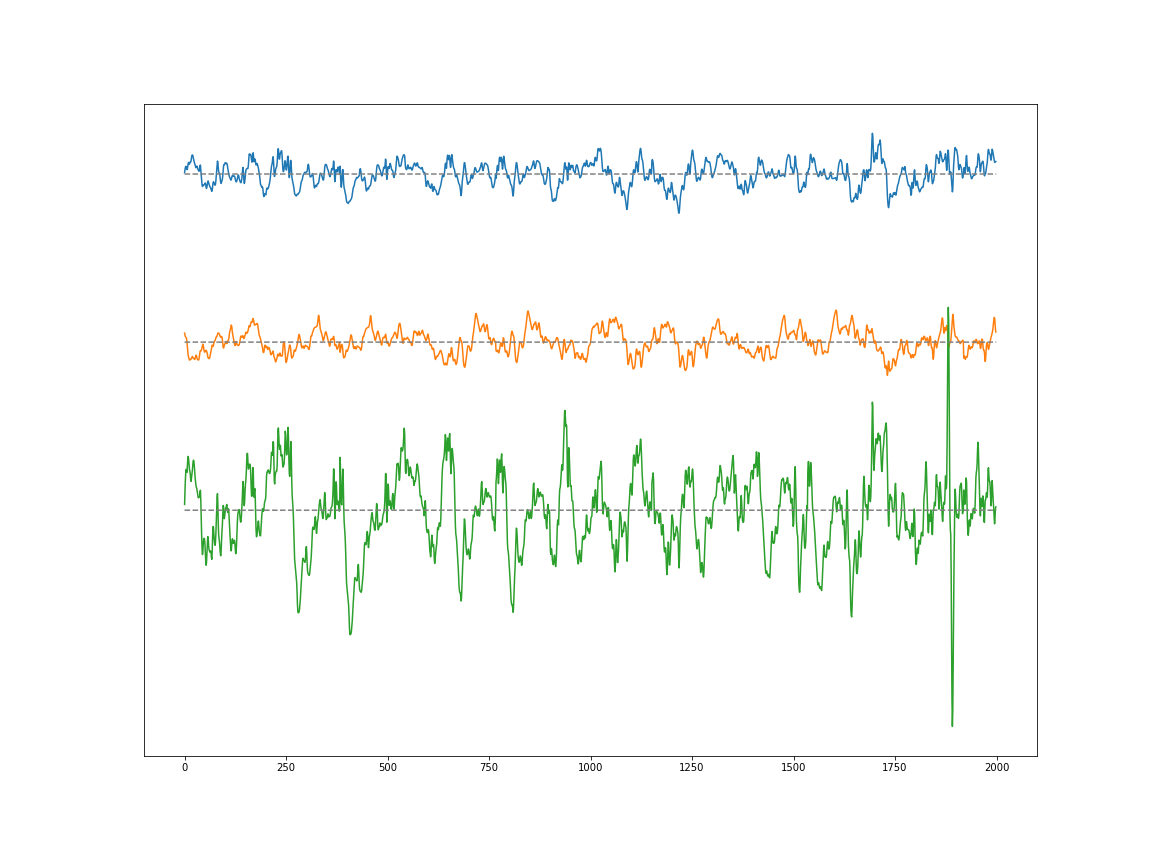} 
}
\subfigure[Original EEG signal with removed three components 9, 20, 36.] {\label{fig:EEG_4}
\includegraphics[width=2.1in]{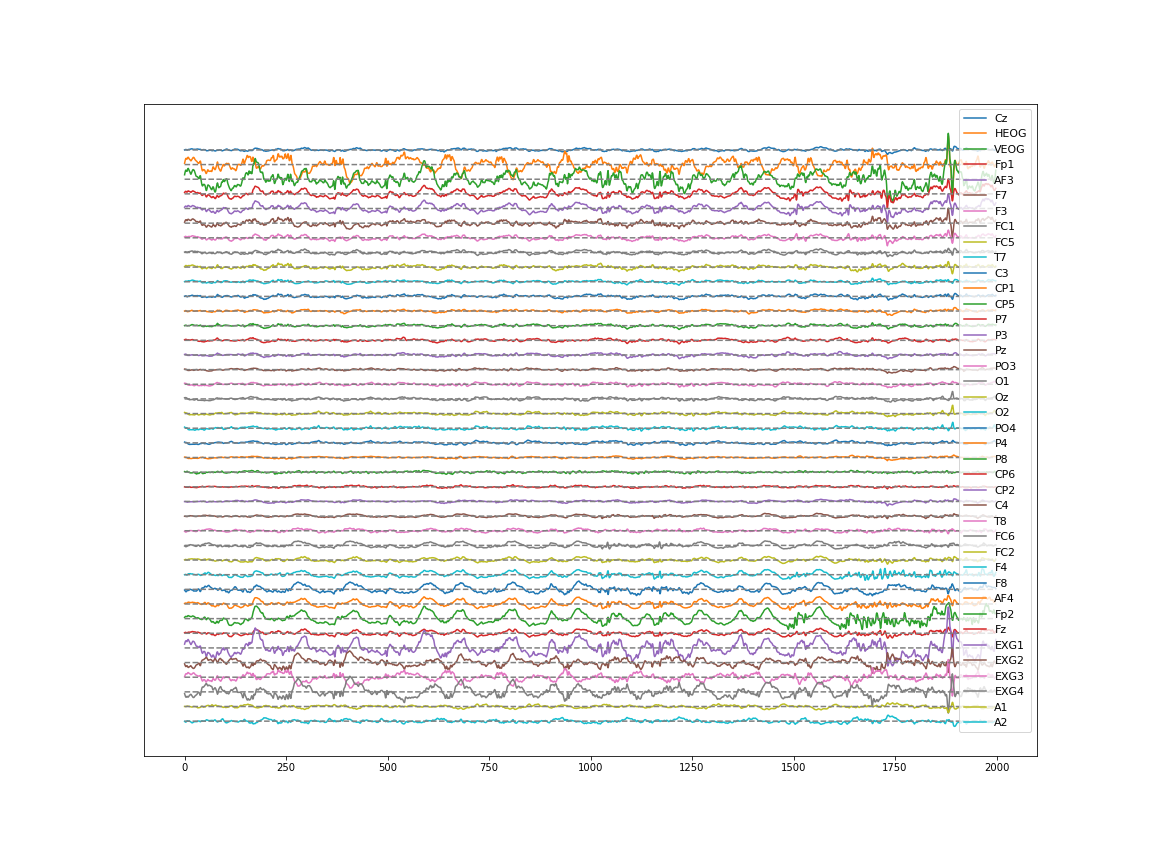}
}
\end{center}
\caption{Results of \ICA in the case of EEG data.}
\label{fig:EEG}
\end{figure*}


\section{Conclusion}

In our work we introduce a new approach to ICA in which we approximate the data density by product of 
Split Generalized Gaussian distribution, which allows us to simultaneously model skewness and heavy-tails in data.
Consequently, we obtain ICA method which gives essentially better results than classical approaches with slightly worst computational complexity.

We verify our approach on images, sound and EEG data.
In the case of source signal reconstructing our approach better recover original signals.
The main reason for this behavior is that real data sets are usually skewed with heavy tails. 
\section{Appendix A}\label{App:A}

\begin{proof}[Proof of Theorem \ref{the:min}.]
Let $X=\{ \x_1, \ldots, \x_n \}$.
We write 
$$
\z_i= W(\x_i-\m), \quad \z_{ij}= \w_j^T(\x_i-\m),
$$
for observation $i$, where $i=1,\ldots,n$ and coordinates $j=1,\ldots,d$.

Let us consider the likelihood function, i.e. 
$$
\begin{array}{l}
L(X;\m,W,\sigma_l,\sigma_r,c) = \prod\limits_{i=1}^{n} SGG_d(\x_i ; \m,W,\sigma_l,\sigma_r,c) \\[6pt] 
=\prod\limits_{i=1}^{n} | \det(W)|  \prod\limits_{j=1}^{d} SGG(  \w_j^T (\x_i - \m) ; 0 , \sigma_l, \sigma_r, c)\\[6pt]
=\Big( c_1|\det(W)| \Big)^{n} \Big( \prod\limits_{j=1}^{d} (\sigma_{lj} + \sigma_{rj}) \Big)^{-n} 
\prod\limits_{i=1}^{n} \prod\limits_{j=1}^{d} \exp \Big[ - \beta^{\frac{c}{2}} \left( \frac{\vert z_{ij} \vert}{\sigma_{lj}} \1_{ \{ z_{ij} \leq 0 \} } + \frac{\vert z_{ij} \vert}{\sigma_{rj}} \1_{ \{ z_{ij} > 0 \} } \right)^c \Big],
\end{array}
$$
where 
$
c_1=\left( \frac{c}{\Gamma(\frac{1}{c})} \sqrt{\beta} \right)^{d}
$
and
$
\beta = \tfrac{\Gamma(\frac{3}{c})}{\Gamma(\frac{1}{c})}.
$
Now we take the log-likelihood function, i.e.
$$
\begin{array}{l}
\ln(L(X;\m,W,\sigma_l,\sigma_r,c)) \\[6pt]
=\ln \bigg( \Big( c_1|\det(W)| \Big)^{n} \Big( \prod\limits_{j=1}^{d} (\sigma_{lj}+\sigma_{rj}) \Big)^{-n} \bigg) + 
 \sum\limits_{i=1}^{n} \sum\limits_{j=1}^{d} \Big[ - \beta^{\frac{c}{2}} \left( \frac{\vert z_{ij} \vert}{\sigma_{lj}} \1_{ \{ z_{ij} \leq 0 \} } + \frac{\vert z_{ij} \vert}{\sigma_{rj}} \1_{ \{ z_{ij} > 0 \} } \right)^c \Big]  \\[6pt]
= \ln \bigg( \Big( c_1|\det(W)| \Big)^{n} \Big( \prod\limits_{j=1}^{d} (\sigma_{lj} + \sigma_{rj}) \Big)^{-n} \bigg)  -
 \beta^{\frac{c}{2}} \sum\limits_{j=1}^{d} \Big( \sigma_{lj}^{-c} \sum\limits_{i \in I_{j}}   \vert z_{ij} \vert^c   + \sigma_{rj}^{-c}  \sum\limits_{i \in I_{j}^{'}}  \vert z_{ij} \vert^c  \Big) \\[6pt]
= \ln \bigg( \Big( c_1|\det(W)| \Big)^{n} \Big( \prod\limits_{j=1}^{d} (\sigma_{lj} + \sigma_{rj}) \Big)^{-n} \bigg)  - 
 \beta^{\frac{c}{2}} \sum\limits_{j=1}^{d} \Big( \sigma_{lj}^{-c} s_{1j}  + \sigma_{rj}^{-c} s_{2j}  \Big).
\end{array}
$$

We fix  $\m$, $W$, $c$ and maximize the log-likelihood function over $\sigma_l$ and $\sigma_r$.
In such a case we have to solve the following system of equations
$$
\begin{array}{l}
\frac{\partial \ln(L(X;\m,W,\sigma_l,\sigma_r,c)) }{\partial \alpha_{lj}} = -\frac{n}{\sigma_{lj} + \sigma_{rj}} + c \beta^{\frac{c}{2}} \sigma_{lj}^{-c-1} s_{1j}
 = 0, \\[6pt] 
 \frac{\partial \ln(L(X;\m,W,\sigma_l,\sigma_r,c)) }{\partial \sigma_{rj}} = -\frac{n}{\sigma_{lj} + \sigma_{rj}} + c \beta^{\frac{c}{2}} \sigma_{rj}^{-c-1} s_{2j}
= 0 , 
\end{array}
$$
for  $ j=1,\ldots,d$.
By simple calculations and substituting $\sigma_{rj} = \sigma_{lj} \left( \frac{s_{2j}}{s_{1j}} \right)^{\frac{1}{c+1}} = \sigma_{lj} \tau$ we obtain the expressions for the estimators
\begin{align*}
\hat{\sigma}_{lj}^c(\m,W) = 
\tfrac{c}{n} \beta^{\frac{c}{2}} s_{1j}^{\frac{c}{c+1}} g_{j}(\m,W), \qquad \hat{\tau}_{j}(\m,W) = \bigg( \frac{s_{2j}}{s_{1j}} \bigg)^{\frac{1}{c+1}}
\end{align*}
and

$$
\hat{\sigma}_{rj}^c(\m,W) = \hat{\sigma}_{lj}^c(\m,W) \cdot \hat{\tau}^c_{j}(\m,W) =
\tfrac{c}{n} \beta^{\frac{c}{2}}  s_{2j}^{\frac{c}{c+1}} g_{j}(\m,W).
$$

Substituting it into the log-likelihood function,
we get
$$
\begin{array}{l}
\hat{L}(\m,W) = \Big( c_1|\det(W)| \Big)^{n} \Big( \prod\limits_{j=1}^{d} (\frac{c}{n})^{\frac{1}{c}} \sqrt{\beta} g_j(\m,W)^{\frac{c+1}{c}} \Big)^{-n} \cdot e^{-\frac{nd}{c}}
\\[6pt]
= \bigg( \frac{nc^{c-1}}{e \Gamma(\frac{1}{c})^c} \bigg)^{\frac{dn}{c}}  \Big( \frac{1}{|\det(W)|^{\frac{c}{c+1}}} \prod\limits_{j=1}^{d} g_j(\m,W) \Big)^{-\frac{n(c+1)}{c}}\\[6pt]
= \bigg( \frac{\kappa n}{c e} \bigg)^{\frac{dn}{c}} \Big( |\det(W)|^{-\frac{c}{c+1}} \prod\limits_{j=1}^{d} g_j(\m,W) \Big)^{-\frac{n(c+1)}{c}}
\end{array}
$$
where $\kappa = \left( \frac{c}{\Gamma(\frac{1}{c})} \right)^c$.
\end{proof}

\section{Appendix B}\label{App:B}

Before we prove Theorem \ref{ther:grad}, we recall the following lemma. 

\begin{lemma}\label{jacobi}
Let $A = (a_{ij})_{1 \leq i,j \leq d}$ be a differentiable map from real numbers to $d \times d$ matrices then
\begin{equation}
\frac{\partial \det(A)}{\partial a_{ij}} = \mathrm{adj}^T(A)_{ij},
\end{equation}
where $\mathrm{adj}(A)$ stands for the adjugate of $A$, i.e. the transpose of the cofactor matrix.
\end{lemma}
\begin{proof}
By the Laplace expansion $\det A = \sum\limits_{j=1}^{d} (-1)^{i+j} a_{ij} M_{ij}$ where $M_{ij}$ is the minor of the entry in the $i$-th row and $j$-th column. Hence
$$\frac{\partial \det A}{\partial a_{ij}} = (-1)^{i+j} M_{ij} = \mathrm{adj}^T(A)_{ij}.$$
\end{proof}

Now we are ready to calculate the gradient of the function $l$.

\begin{proof}[Proof of Theorem \ref{ther:grad}.]
Let us start with the partial derivative of $\ln l$ with respect to $\m$. We have
$$
\begin{array}{l}
\frac{\partial \ln {l}(X;\m,W,c)}{\partial \m_k} =
\sum \limits_{j=1}^d \frac{\partial \ln ({g}_j(\m,W))}{\partial \m_k} = \sum\limits_{j=1}^d \frac{1}{{s}_{1j}^{\frac{1}{c+1}} + {s}_{2j}^{\frac{1}{c+1}}} \frac{\partial \left( {s}_{1j}^{\frac{1}{c+1}} + {s}_{2j}^{\frac{1}{c+1}} \right)}{\partial \m_k} =\\
 \sum \limits_{j=1}^d \frac{1}{{s}_{1j}^{\frac{1}{c+1}} + {s}_{2j}^{\frac{1}{c+1}}} \bigg(
\frac{1}{(c+1) {s}_{1j}^{\frac{c}{c+1}}} \frac{\partial {s}_{1j}}{\partial \m_k} +
\frac{1}{(c+1) {s}_{2j}^{\frac{c}{c+1}}} \frac{\partial {s}_{2j}}{\partial \m_k}
\bigg).
\end{array}
$$
Now, we need $\frac{\partial {s}_{1j}}{\partial \m_k}$ and $\frac{\partial {s}_{2j}}{\partial \m_k}$, therefore
$$
\begin{array}{l}
\frac{\partial {s}_{1j}}{\partial \m_k} = 
\sum\limits_{i \in {I}_j} \frac{\partial \vert \omega^T_j (\x_i - \m) \vert^c}{\partial \m_k} = \sum\limits_{i \in {I}_j} -c \vert \omega^T_j (\x_i - \m) \vert^{c-1} \frac{\partial ( \omega^T_j (\x_i - \m) )}{\partial \m_k} = \\
 \sum\limits_{i \in {I}_j} c \vert \omega^T_j (\x_i - \m) \vert^{c-1} \omega_{jk} = \sum\limits_{i \in {I}_j} -c (-1)^{c-1} \left( \omega^T_j (\x_i - \m) \right)^{c-1} \omega_{jk}.
\end{array}
$$
Analogously we get
$$
\begin{array}{l}
\frac{\partial {s}_{2j}}{\partial \m_k} = \sum\limits_{i \in {I}_j^{'}} -c \vert \omega^T_j (\x_i - \m) \vert^{c-1} \omega_{jk} = \sum\limits_{i \in {I}_j^{'}} -c \left( \omega^T_j (\x_i - \m) \right)^{c-1} \omega_{jk}.
\end{array}
$$
Hence 
$$
\begin{array}{l}
\frac{\partial \ln {l}}{\partial \m_k} =\sum\limits_{j=1}^d \frac{1}{{s}_{1j}^{\frac{1}{c+1}} + {s}_{2j}^{\frac{1}{c+1}}} \bigg(
\frac{1}{(c+1) {s}_{1j}^{\frac{c}{c+1}}} \sum\limits_{i \in I_j} c \vert \omega^T_j (\x_i - \m) \vert^{c-1} \omega_{jk} - \\[6pt]
\frac{1}{(c+1) {s}_{2j}^{\frac{c}{c+1}}} \sum\limits_{i \in I_j^{'}} c \vert \omega^T_j (\x_i - \m) \vert^{c-1} \omega_{jk}
\bigg).
\end{array}
$$

Now we calculate the partial derivative of $\ln {l}(X;\m,W,c)$ with respect to the matrix $W$. We have
$$
\begin{array}{l}
\frac{\partial \ln {l}(X;\m,W,c)}{\partial \w_{pk}} = \frac{\partial \ln |\det(W)|^{-\frac{c}{c+1}}}{\partial \w_{pk}} + \sum\limits_{j=1}^d \frac{\partial \ln ({g}_j(\m,W))}{\partial \w_{pk}}.
\end{array}
$$
To calculate the derivative of the determinant we use Jacobi's formula (see Lemma \ref{jacobi}).
Hence
$$
\begin{array}{l}
\frac{\partial \ln (\det(W)^{-\frac{c}{c+1}})}{\partial \w_{pk}} = \det(W)^{\frac{c}{c+1}}  \Big(-\frac{c}{c+1}\Big)  \det(W)^{-\frac{2c+1}{c+1}}  \frac{\partial \det(W)}{\partial \w_{pk}} = -\frac{c}{c+1} \det(W)^{-1}  \mathrm{adj}^T(W)_{pk} \\[6pt]
 = -\frac{c}{c+1} \frac{1}{\det(W)}  \left[\det(W)  (W^{-1})^T_{pk}\right]= -\frac{c}{c+1}  (\w^{-1})^T_{pk},
\end{array}
$$
where $(\w^{-1})^T_{pk}$ is the element in the $p$-th row and $k$-th column of the matrix $(W^{-1})^T$. Now we calculate 
$$
\begin{array}{l}
\frac{\partial \ln ({g}_j(\m,W))}{\partial \w_{pk}} = \frac{1}{{s}_{1j}^{\frac{1}{c+1}} + {s}_{2j}^{\frac{1}{c+1}}} \frac{\partial \left({s}_{1j}^{\frac{1}{c+1}} + {s}_{2j}^{\frac{1}{c+1}} \right)}{\partial \w_{pk}}= \frac{1}{{s}_{1j}^{\frac{1}{c+1}} + {s}_{2j}^{\frac{1}{c+1}}} \bigg(
\frac{1}{(c+1) {s}_{1j}^{\frac{c}{c+1}}}  \frac{\partial {s}_{1j}}{\partial \w_{pk}} +
\frac{1}{(c+1) {s}_{2j}^{\frac{c}{c+1}}}  \frac{\partial {s}_{2j}}{\partial \w_{pk}}
\bigg),
\end{array}
$$
where
$$
\begin{array}{l}
\frac{\partial {s}_{1j}}{\partial \w_{pk}} = \sum\limits_{ i \in {I}_j} \frac{\partial \vert \w^T_j (\x_i - \m) \vert^c}{\partial \w_{pk}} = \sum\limits_{ i \in {I}_j} -c \vert \w^T_j (\x_i - \m) \vert^{c-1} \frac{\partial \w^T_j (\x_i - \m)}{\partial \w_{pk}}=
\\[6pt]
\left\{ \begin{array}{ll}
0, & \text{if} \; j\neq p\\
\sum\limits_{ i \in {I}_p} -c (-1)^{c-1} \left( \w^T_p (\x_i - \m) \right)^{c-1} (\x_{ik} - \m_k), & \text{if} \; j=p\\
\end{array} \right.
\end{array}
$$
and $\x_{ik}$ is the $k$-th element of the vector $\x_i$. Analogously we get
$$\frac{\partial {s}_{2j}}{\partial \w_{pk}} = \left\{ \begin{array}{ll}
0, & \text{if} \; j\neq p\\
\sum\limits_{ i \in {I}_p^{'}} c \vert \w^T_p (\x_i - \m) \vert^{c-1} (\x_{ik} - \m_k), & \text{if} \; j=p.
\end{array} \right.
$$
Hence we obtain 
$$
\begin{array}{l}
\frac{\partial \ln {l}}{\partial \w_{pk}} = -\frac{c}{c+1}  (\w^{-1})^T_{pk} + \frac{1}{{s}_{1p}^{\frac{1}{c+1}} +{s}_{2p}^{\frac{1}{c+1}}} 
 \bigg(
\frac{-1}{c+1} {s}_{1p}^{-\frac{c}{c+1}} \sum\limits_{ i \in {I}_p} c \vert \w^T_p (\x_i - \m) \vert^{c-1} (\x_{ik} - \m_k)\\[6pt]
+ \frac{1}{c+1} {s}_{2p}^{-\frac{c}{c+1}} \sum\limits_{ i \in {I}_p^{'}} c \vert \w^T_p (\x_i - \m) \vert^{c-1} (\x_{ik} - \m_k) \bigg).
\end{array}
$$
Now we calculate the derivative with respect to $c$.
$$
\begin{array}{l}
\frac{\partial \ln {l}(X;\m,W)}{\partial c} =
-\frac{\partial}{\partial c}\big(\frac{c}{c+1} \ln |\det(W)| \big) + \sum \limits_{j=1}^d \frac{\partial \ln ({g}_j(\m,W))}{\partial c} =\\
-\frac{1}{(c+1)^2}  \ln |\det(W)| + \sum \limits_{j=1}^d \frac{1}{{s}_{1j}^{\frac{1}{c+1}} + {s}_{2j}^{\frac{1}{c+1}}} \frac{\partial}{\partial c} \big( {s}_{1j}^{\frac{1}{c+1}} + {s}_{2j}^{\frac{1}{c+1}} \big) =\\
-\frac{1}{(c+1)^2}  \ln |\det(W)| + \sum \limits_{j=1}^d \frac{1}{{s}_{1j}^{\frac{1}{c+1}} + {s}_{2j}^{\frac{1}{c+1}}} \bigg( {s}_{1j}^{\frac{1}{c+1}} \frac{\partial}{\partial c} (\frac{1}{c+1} \ln {s}_{1j}) + {s}_{2j}^{\frac{1}{c+1}} \frac{\partial}{\partial c} (\frac{1}{c+1} \ln {s}_{2j})  \bigg)\\
= -\frac{1}{(c+1)^2}  \ln |\det(W)| + \\
\sum \limits_{j=1}^d \frac{1}{{s}_{1j}^{\frac{1}{c+1}} + {s}_{2j}^{\frac{1}{c+1}}} \bigg( -\frac{{s}_{1j}^{\frac{1}{c+1}}}{(c+1)^2} \ln {s}_{1j} + \frac{1}{c+1}s_{1j}^{-\frac{c}{c+1}} \frac{\partial s_{1j}}{\partial c} -\frac{{s}_{2j}^{\frac{1}{c+1}}}{(c+1)^2} \ln {s}_{2j} + \frac{1}{c+1}s_{2j}^{-\frac{c}{c+1}} \frac{\partial s_{2j}}{\partial c} \bigg).
\end{array}
$$
where
$$
\begin{array}{c}
\frac{\partial {s}_{1j} }{\partial c} = \! \sum\limits_{i \in I_j} \vert \w_{j}^T (\x_i-\m) \vert^c \ln \vert \w_{j}^T (\x_i-\m) \vert,
\\[1ex]
\frac{\partial {s}_{2j} }{\partial c} = \! \sum\limits_{i \in I_j^{'}} \vert \w_{j}^T (\x_i-\m) \vert^c \ln \vert \w_{j}^T (\x_i-\m) \vert.
\end{array}
$$
\end{proof}


\section*{References}

\end{document}